\documentclass{article}

\usepackage{microtype}
\usepackage{graphicx}
\usepackage{xcolor}
\usepackage{xparse}  
\usepackage{subfig}
\usepackage{booktabs} 
\usepackage{bm}
\usepackage{pdfpages}
\usepackage{dsfont}

\usepackage{amsmath,amsthm}
\usepackage{amssymb}
\usepackage{mathtools}

\usepackage{booktabs}
\usepackage{longtable}
\usepackage{array}
\usepackage{multirow}
\usepackage{wrapfig}
\usepackage{float}
\usepackage{colortbl}
\usepackage{pdflscape}
\usepackage{tabu}
\usepackage{threeparttable}
\usepackage{threeparttablex}
\usepackage[normalem]{ulem}
\usepackage{makecell}
\usepackage{xcolor}

\usepackage[round]{natbib}
\usepackage{paralist}

\usepackage{hyperref}

\usepackage{cleveref}

\crefname{supp}{Supplement}{Supplements}

\DeclareMathOperator*{\argmax}{arg\,max}

\DeclareMathOperator{\CE}{CE}
\DeclareMathOperator{\ECE}{ECE}

\theoremstyle{plain}
\newtheorem{theorem}{Theorem}
\theoremstyle{plain}
\newtheorem{proposition}[theorem]{Proposition}
\theoremstyle{definition}
\newtheorem{definition}[theorem]{Definition}
\theoremstyle{plain}
\newtheorem{assumption}{Assumption}

\crefname{assump}{Assumption}{Assumptions}

\usepackage{xr}
\makeatletter
\newcommand*{\addFileDependency}[1]{
  \typeout{(#1)}
  \@addtofilelist{#1}
  \IfFileExists{#1}{}{\typeout{No file #1.}}
}
\makeatother

\newcommand*{\myexternaldocument}[1]{%
    \externaldocument{#1}%
    \addFileDependency{#1.tex}%
    \addFileDependency{#1.aux}%
}
\myexternaldocument{supp}

\usepackage[accepted]{icml2021}

\icmltitlerunning{Meta-Cal: Well-controlled Post-hoc Calibration by Ranking}

\begin{document}

\twocolumn[
\icmltitle{Meta-Cal: Well-controlled Post-hoc Calibration by Ranking}




\begin{icmlauthorlist}
\icmlauthor{Xingchen Ma}{kul}
\icmlauthor{Matthew B.\ Blaschko}{kul}
\end{icmlauthorlist}

\icmlaffiliation{kul}{ESAT-PSI, KU Leuven, Belgium}

\icmlcorrespondingauthor{Xingchen Ma}{xingchen.ma@esat.kuleuven.be}

\icmlkeywords{Calibration, Ranking}

\vskip 0.3in
]



\printAffiliationsAndNotice{}  

\begin{abstract}
In many applications, it is desirable that a classifier not only makes accurate predictions, but also outputs calibrated posterior probabilities. However, many existing classifiers, especially deep neural network classifiers, tend to be uncalibrated. Post-hoc calibration is a technique to recalibrate a model by learning a calibration map. Existing approaches mostly focus on constructing calibration maps with low calibration errors, however, this quality is inadequate for a calibrator being useful. In this paper, we introduce two constraints that are worth consideration in designing a calibration map for post-hoc calibration. Then we present Meta-Cal, which is built from a base calibrator and a ranking model. Under some mild assumptions, two high-probability bounds are given with respect to these constraints. Empirical results on CIFAR-10, CIFAR-100 and ImageNet and a range of popular network architectures show our proposed method significantly outperforms the current state of the art for post-hoc multi-class classification calibration.
\end{abstract}

\section{Introduction}
\label{sec:introduction}

Recent advances in machine learning 
have resulted in very high accuracy in many classification tasks. 
In the context of computer vision, there has been a steady increase in \mbox{top-1} accuracy on ImageNet \citep{russakovsky2015} since AlexNet \citep{krizhevsky2012}.
While highly accurate models are desirable in general, in some applications, we want that the output of a classifier is a calibrated posterior probability. This is especially important in cost-sensitive classification tasks, such as medical diagnosis \citep{ma2017} and autonomous driving \citep{chen2015b}. 

The goal of classification calibration is to obtain a probabilistic model that has low calibration error. A formal definition of calibration error will be given in \Cref{sec:problem-formulation}. 
To achieve this objective, one can design models with intrinsically low calibration errors. These models \citep{wilson2016,pereyra2017,lakshminarayanan2016,malinin2018,milios2018,maddox2019} are usually developed with a Bayesian nature and tend to be expensive to train and make inferences. On the other end of the spectrum, a post-hoc calibration method transforms outputs of an existing classification model into well-calibrated predictions by learning a trainable post-processing step. 
A lot of effort~\citep{zhang2020,guo2017,ding2020,patel2020,wenger2020,kull2019} has been done along this line given the design difficulties and computational burdens of the former approach. In this work, we focus on post-hoc multi-class calibration.

Among the existing literature on post-hoc calibration, some methods \citep{zhang2020,guo2017,ding2020,patel2020} seek calibration maps that preserve classification accuracy, while some other works \citep{wenger2020,kull2019} focus more on calibration error without explicitly enforcing  accuracy-preservation. Compared with the latter, there is no potential accuracy drop in the former, but its family of calibration maps is less flexible. We will formalize the limitation of such a family by giving a lower bound of its calibration error in \Cref{prop:acc-pre-ece-lower-bound}.
In this work, we try to get the best of both worlds by combining an existing calibration model with a bipartite ranking model, and we call our proposed post-hoc calibration method \textit{Meta-Cal}.

Similar to previous works~\citep{wenger2020,kull2019}, Meta-Cal does not enforce the overall accuracy being kept. As we will show later in~\Cref{sec:method}, such a calibrator can be of little importance even if its calibration error is low. Additional constraints should be taken into consideration. In this work, we introduce two constraints: 
\begin{inparaenum}[(i)]
\item miscoverage rate,
\item coverage accuracy.
\end{inparaenum}
In~\Cref{sec:method:miscoverage-control} and~\Cref{sec:method:coverage-accuracy-control}, we show Meta-Cal has full control over these two constraints. 
An intuitive interpretation of the above two constraints is by considering a quality assurance (QA) system. In such a system, two desired properties are: 
\begin{inparaenum}[(i)]
\item for products accepted by the QA system, their quality requirements are satisfied with high confidence,
\item the QA system does not perform unnecessary rejection.
\end{inparaenum}
The first point corresponds to the coverage accuracy of a calibration map and the second one corresponds to the miscoverage rate. Their formal definitions are given in \Cref{sec:method}. Our proposed Meta-Cal is constructed from a base calibrator and a ranking model. If this base calibration model is accuracy-preserving, it is guaranteed that Meta-Cal has an improved calibration error bound over it.


The main contributions of this paper are:
\begin{itemize}
    \item A novel post-hoc calibration approach (Meta-Cal) for multi-class classification is proposed. Meta-Cal augments a base calibration model and obtains better calibration performance.
    \item Two practical constraints are investigated alongside Meta-Cal. We show how these constraints are incorporated in constructing our proposed calibrator. Theoretical results on high-probability bounds w.r.t. these constraints are presented. 
    \item We show the effectiveness of our proposed approach and validate our theoretical claims through a series of empirical experiments.
\end{itemize}

In the next section, we briefly review post-hoc calibration methods for classification and bipartite ranking models.
Necessary background, notation and assumptions that will be used throughout this paper are given in~\Cref{sec:problem-formulation}. \Cref{sec:method} describes our proposed approach. 
Two practical constraints of Meta-Cal are discussed in~\Cref{sec:method:miscoverage-control} and~\Cref{sec:method:coverage-accuracy-control} respectively. 
Empirical results are presented in~\Cref{sec:experiment}. Finally, we conclude in~\Cref{sec:conclusion}.

\section{Related Work}
\label{sec:related-work}

\textbf{Post-hoc Calibration of Classifiers:} Platt scaling \citep{platt1999,lin2007} is proposed to learn a calibration transformation to map the outputs of a binary SVM into posterior probabilities. Extensions of Platt scaling for multi-class calibration include temperature scaling \citep{guo2017}, ensemble temperature scaling \citep{zhang2020} and local temperature scaling \citep{ding2020}. 
Different from the parametric approach adopted by Platt scaling, histogram binning \citep{zadrozny2001a} is a non-parametric post-hoc calibration method in binary settings. Two popular refinements of histogram binning are isotonic regression \citep{zadrozny2002} and Bayesian binning into quantiles \citep{naeini2015}. 
Platt scaling assumes scores of each class are Gaussian distributed, however, this assumption is hardly satisfied for many probabilistic classifiers. Motivated by this assumption violation, Beta calibration~\citep{kull2017} is proposed and it assumes Beta distribution of scores within each class.
For multi-class classification, Dirichlet calibration \citep{kull2019} is proposed as an extension of Beta calibration. 
More recently, a Gaussian Process based calibration method is presented in~\citet{wenger2020}.

\textbf{Bipartite Ranking:} The problem of bipartite ranking has been studied for a long time. In bipartite ranking, one wants to compare or rank two different objects and decide which one is better. \citet{burges2005} uses a neural network to model a ranking function and trains this network by gradient decent methods. \citet{clemencon2008} defines a statistical framework for such ranking problems and provides several consistency results. \citet{narasimhan2013} investigates the relationship between binary classification and bipartite ranking. Our work is closely related to \citet{narasimhan2013} and we rely on the theoretical results presented there to construct a binary class probability estimation (CPE) model from a ranking model.

\textbf{Selective Classification:} Selective classification is a technique that augments a classifier with a reject option. The essence of selective classification is the trade-off between coverage and accuracy. \citet{el-yaniv2010} lays the foundation for selective classification by characterizing the theoretical and practical boundaries of risk-coverage trade-offs. \citet{geifman2017} applies selective classification to deep neural networks. While the starting point of our work is to make improvements in post-hoc classification calibration under several practical constraints, it turns out that techniques present in this work can also be used in selective classification. \citet{geifman2017} gives a one-sided and tight high-probability risk bound, while we present two high probability bounds for miscoverage rate and coverage accuracy. 
It should be noted that definitions of risk and coverage in our setting are different from these two metrics defined in the context of selective classification~\citep{el-yaniv2010,geifman2017}, thus these bounds are not directly comparable.

\section{Problem Formulation}
\label{sec:problem-formulation}

In this section, we summarize some necessary background, notation and assumptions that will be used throughout this paper. 
Let $\mathcal{X}$ be the input space and $\mathcal{Y}=[k]\coloneqq \{ 1,\cdots,k \}$ be the label space in multi-class classification. We denote the $(k-1)$-simplex by $\Delta^k \coloneqq \{ (p_1,\cdots,p_k) \mid \sum_{i \in [k]} p_{i} = 1, p_i \in [0,1] \}$.
Suppose a probabilistic classifier $f: \mathcal{X} \to \Delta^k$ is given and this classifier is trained on an i.i.d.\ data set following a joint distribution $\mathbb{P}(X,Y)$ on $\mathcal{X} \times \mathcal{Y}$.
Let the random variable $X$ take values in input space $\mathcal{X}$, $Z=(Z_1,\cdots,Z_k)$ be the output of $f$ applied to $X$, $\hat{Z} = \max_{i \in [k]} Z_i$ be the confidence score and $\hat{Y} = \argmax_{i \in [k]} Z_i$ be the prediction of $X$. Whenever there is a tie, it is broken uniformly at random.
To measure the level of calibration, following \citet{naeini2015,kumar2019,wenger2020}, the calibration error is defined in~\Cref{def:calibration-err}. 

\begin{definition}[Calibration error]
\label{def:calibration-err}
The $L_p$ calibration error of $f$ with $p \ge 1$ is:
\begin{equation}
\CE_p (f) = \mathbb{E} \left[  \left| \mathbb{E} \left( \mathds{1}(Y=\hat{Y}) \mid \hat{Z} \right) - \hat{Z} \right|^{p}  \right]^{1/p},
\end{equation}
where the expectation is taken with respect to $\mathbb{P}(X,Y)$.
\end{definition}
If the calibration error of a model $f$ is zero, we say $f$ is perfectly calibrated. In practice, $\CE_p(\cdot)$ is unobservable since it depends on the unknown joint distribution $\mathbb{P}(X,Y)$. To empirically estimate the calibration error based on a finite data set, prior works apply a fixed binning scheme $B \in \mathcal{B}$, where $\mathcal{B}$ is the family of binning schemes, on values of $\hat{Z}$ and calculate the difference between the accuracy and the average of confidence scores in every bin. An example family for a binning scheme is a partition of the interval $[0,1]$.
When $p=1$, the calibration error is also known as \textit{expected calibration error} (ECE) \citep{guo2017} and an empirical estimator based on $B$ is denoted by $\widehat{\ECE}_{B}$. Although it is known that $\widehat{\ECE}_{B}$ is a biased estimation~\cite{kumar2019,zhang2020} of $\ECE$ and its magnitude can be hard to interpret, in this work, we use this binned estimator to measure the level of calibration due to its simplicity and popularity. 

In the post-hoc calibration problem, one wants to find a function that transforms the outputs of $f$ to make the final model better calibrated. Formally, a calibration map is a function $g: \Delta^k \to \Delta^k$. 
The goal is to learn a mapping $g \in \mathcal{G}$ on a finite calibration data set $\{(x_i, y_i)\}_{i=1}^n$ that is drawn i.i.d.\ from the joint distribution $\mathbb{P}(X,Y)$ so that the composition $g \circ f$ has a small calibration error, where $\mathcal{G}$ is the calibration family.

In the problem of bipartite ranking, we want to find a ranking model $h: \mathcal{X} \to \mathbb{R}$ that ranks positive examples and negative ones so that the positive examples have higher scores with a high probability. Formally, one wants to minimize the following \textit{ranking risk} \citep{narasimhan2013,clemencon2008}:
\begin{equation*}
    L(f) = \mathbb{P}( (\epsilon - \epsilon') \cdot (h(X) - h(X')) ),
\end{equation*}
where $(X,\epsilon), (X', \epsilon')$ are i.i.d. pairs taking values in $\mathcal{X} \times \{-1, +1\}$. The ranking risk is the probability that $h$ ranks two randomly drawn pairs incorrectly~\cite{clemencon2008}. In this paper, if not stated otherwise, the random variable $\epsilon$ is a sign variable used to indicate whether $X$ is correctly classified or not by $f$, that is, $\epsilon = 2 \cdot \mathds{1}(Y \neq \hat{Y}) - 1$. In general, a bipartite ranking model is learnt using a consistent pairwise surrogate loss, such as exponential loss or logistic loss~\cite{gao2014}.

In the following, we list some assumptions that will be used in~\Cref{sec:method}.

\begin{assumption}
\label[assump]{assump:h-continuous}
Denote the marginal distribution of $X$ taking values in $\mathcal{X}$ by $\mathbb{P}(X)$. The induced distribution of $h(X)$ is continuous, where $h$ is a ranking model.
\end{assumption}

\begin{assumption}
\label[assump]{assump:eta-square-integral}
Let $h: \mathcal{X} \to [a,b]$ (where $a,b \in \mathbb{R}$, $a < b$) be any bounded-range ranking model. $\eta_h(s) = \mathbb{P}(\epsilon=1 \mid h(x) = s)$ is square-integrable w.r.t. the induced density of $h(X)$.
\end{assumption}

The first assumption simply means for two independent samples from $\mathbb{P}(X)$, the probability of their ranking scores being equal vanishes. The second assumption is required by~\citet{narasimhan2013} as we rely on their results.

\section{Meta-Cal}
\label{sec:method} 

In this section, we start by showing that the family of accuracy-preserving calibration maps has an inherent limitation (\Cref{prop:acc-pre-ece-lower-bound}) and this motivates us to combine a bipartite ranking model with an existing calibration model. 
Then we demonstrate that models with low calibration errors alone do not necessarily indicate they are practical. Post-hoc calibration should be considered together with other factors. In this work, we investigate two useful constraints in~\Cref{sec:method:miscoverage-control} and~\Cref{sec:method:coverage-accuracy-control}, respectively: Miscoverage rate control and coverage accuracy control.

\begin{proposition}[Lower bound of ECE]
\label{prop:acc-pre-ece-lower-bound}
Define $\mathcal{G}_a$ as the family of accuracy-preserving calibration maps, that is,
\begin{equation*}
    \mathcal{G}_a = \{ g \in \mathcal{G}: \argmax_{i \in [k]} f(X)_i = \argmax_{i \in [k]} g(f(X))_i\}.
\end{equation*}
Then for all $g \in \mathcal{G}_a$
\begin{equation}
    \sup_{B \in \mathcal{B}} \widehat{\ECE}_B(g \circ f )  > \frac{1 - \hat{\pi}_0}{k} ,
\end{equation}
where $\widehat{\ECE}_{B}$ is estimated on a finite data set and $\hat{\pi}_0$ is the empirical accuracy of $f$ on the same data set. 
Further we can show $\forall B \in \mathcal{B}, \forall g \in \mathcal{G}$, 
\begin{equation}
    \ECE(g \circ f)  \ge \widehat{\ECE}_B(g \circ f ).
\end{equation}
\end{proposition}

All proofs are provided in~\Cref{supp:proofs} and we only describe a sketch of the proof here. The first step is to find the supremum of the binned estimator $\widehat{\ECE}_B(g \circ f)$ across all calibration maps $g \in \mathcal{G}_a$ and all binning schemes $B \in \mathcal{B}$. This is achieved using Minkowski's inequality. The second step is showing this supremum is lower bounded by $(1 - \hat{\pi}_0) / k$. Unless $f$ is perfectly accurate, $(1 - \hat{\pi}_0) / k$ will be strictly positive. The last step is already given in \citet{kumar2019} by applying Jensen's inequality. We note bounding $\widehat{\ECE}_B(g \circ f)$ instead of its supremum makes little sense here since it is easy to construct a trivial $g \in \mathcal{G}_a$ and a binning scheme $B$ such that $\widehat{\ECE}_B(g \circ f)$ is exactly zero. In \Cref{supp:trivial-g}, we give such a trivial construction. \Cref{prop:acc-pre-ece-lower-bound} makes it clear that an accuracy-preserving calibration has an inherent limitation. 
Such a calibration map cannot perfectly calibrate a classifier even if it is given the label information. To avoid potential confusion, we emphasize here the empirical estimator of ECE measures the \textit{top-label} calibration error instead of the calibration error w.r.t.\ a fixed class.


In the following proposition, 
we show if a post-hoc calibration map is allowed to change the accuracy of the original classifier, it can achieve perfect calibration with the aid of a binary classifier.

\begin{proposition}[Optimal calibration map]
\label{prop:optimal-calibration}
Suppose realizations of $X \mid Y \neq \hat{Y}$ with positive labels and realizations of $X \mid Y = \hat{Y}$ with negative labels are separable and we are given a perfect binary classifier $\phi^* : \mathcal{X} \to \{-1, +1\}$ that is able to classify these two sets. For a new observation $x \sim X$, an optimal calibration map $g^* \in \mathcal{G}$ that minimizes $\sup_{B \in \mathcal{B}, g \in \mathcal{G}} \widehat{\ECE}_B(g \circ f )$ can be constructed using the following rules:
\begin{itemize}
    \item if $\phi^{*}(x)=-1$, we let $\max_{i \in [k]} g^{*}(x)_i = 1$
    \item if $\phi^{*}(x)=+1$, we let $\max_{i \in [k]} g^{*}(x)_i = \frac{1}{k}$
\end{itemize}
\end{proposition}

\begin{proof}
The optimality of $g^*$ follows directly from the proof steps of~\Cref{prop:acc-pre-ece-lower-bound}.
\end{proof}


We note $g^* \not\in \mathcal{G}_a$, since the used tie-breaking strategy is random at uniform, thus $g^*$ does not preserve the accuracy after the above calibration. It is straightforward to check the constructed $g^*$ in~\Cref{prop:optimal-calibration} has zero calibration error, that is $\ECE(g^* \circ f) =  \sup_{B \in \mathcal{B}} \widehat{\ECE}_B(g^* \circ f ) = 0$. 

However, it is unlikely that a perfect binary classifier can be obtained in practice and classification errors will occur when $\phi(X) \neq 2 \cdot \mathds{1}(Y \neq \hat{Y}) -1$. We denote the population \mbox{Type I error} (false positive error) and \mbox{Type II error} (false negative error) of $\phi$ by $R_0(\phi)$ and $R_1(\phi)$ respectively. In the following, we define the \textit{miscoverage rate} and \textit{coverage accuracy} of a calibration map $g$ constructed using the rules in~\Cref{prop:optimal-calibration}. These two constraints will be used to construct the final calibration map.

\begin{definition}[Miscoverage rate]
\label{def:miscoverage-rate}
The miscoverage rate $F_0(g)$ of $g$ constructed using the rules in~\Cref{prop:optimal-calibration} is defined to be the \mbox{Type I error} of the binary classifier $\phi$ associated with $g$.  
\end{definition}

\begin{definition}[Coverage accuracy]
\label{def:coverage-accuracy}
The coverage accuracy $F_1(g)$ of $g$ constructed using the rules in~\Cref{prop:optimal-calibration} is defined to be the precision of the binary classifier $\phi$ associated with $g$. 
\end{definition}

The rational behind the above definitions will be illustrated by two examples later in this section. To be brief, these two constraints are necessary for a calibrator to be useful, if this calibrator cannot preserve the accuracy.
The miscoverage rate of $g$ is the proportion of instances whose predictions are no longer correct after the calibration to instances whose predictions are originally correct before the calibration. The coverage accuracy of $g$ is the accuracy among covered instances (i.e.\ those instances classified as negative by the binary classifier) after the calibration. The correspondence between the \mbox{Type I error} (precision) of a binary classifier and the miscoverage rate (coverage accuracy) of a calibration map is due to the construction rules in~\Cref{prop:optimal-calibration}.

\begin{proposition}
\label{prop:ece-naive-binary-classifier}
Using the construction rules in~\Cref{prop:optimal-calibration}, we can estimate the expected calibration error of $g$ on a finite data set:
\begin{equation*}
\sup_{B \in \mathcal{B}} \widehat{\ECE}_B(g \circ f )  = \hat{R}_1(\phi) \cdot (1 - \hat{\pi}_0),
\end{equation*}
where $\hat{R}_1(\phi)$ is the empirical \mbox{Type II error} of $\phi$ which is trained following~\Cref{prop:optimal-calibration}, $g$ is constructed using the rules in~\Cref{prop:optimal-calibration} and $\hat{\pi}_0$ is the empirical accuracy of $f$ on this finite data set.
\end{proposition}

The following two trivial calibration maps are used to illustrate the necessity of the above two constraints. These two calibrators have low calibration errors, but they are hardly useful.

\textbf{High miscoverage rate:} Based on the results in~\Cref{prop:ece-naive-binary-classifier}, to obtain a low calibration error, we could build a binary classifier with a low \mbox{Type II error}. This classifier has a potentially high \mbox{Type I error}, thus the miscoverage rate is potentially high as well. For example, we can select a small portion of negative instances and relabel the remaining instances as positive. Then a $1$-nearest neighbor classifier is such a classifier.

\textbf{Low coverage accuracy:} Regardless of the input, we let the output of a calibration map to be the marginal class distribution. The calibration error of this calibration map is equal to 0, but its coverage accuracy is low and equal to the proportion of the largest class.

Obviously, the above two constructions are of little consequence in practice. 
Another motivating example demonstrating the usefulness these two constraints is an automated decision model with human-in-the-loop (HIL). In such a model, on one hand, it is desirable that human interaction is as minimal as possible (low miscoverage rate). On the other hand, we hope this model can make accurate and reliable decisions (high coverage accuracy).
We now describe how to construct a calibration map with a controlled miscoverage rate or a controlled coverage accuracy.\footnote{Code is available at \url{https://github.com/maxc01/metacal}}

\subsection{Miscoverage Rate Control}
\label{sec:method:miscoverage-control}

To control the miscoverage rate of our proposed calibration method, we need to construct a suitable binary classifier with a controlled $\mbox{Type I error}$. Such a classifier can be built using the following steps. 
Given a finite calibration data set $\{(x_i, y_i)\}_{i=1}^n$ that is drawn i.i.d.\ from the joint distribution $\mathbb{P}(X,Y)$ on $\mathcal{X} \times \mathcal{Y}$, we first create a binary classification data set $\{(x_i, 2 \cdot \mathds{1}(y_i \neq \hat{y}_i) -1 )\}_{i=1}^n$, where $\hat{y}_i = \argmax_{i \in [k]} f(x_i)$ and $f$ is a multi-class classifier to be calibrated. 
Without loss of generality, we suppose the first $n_1$ inputs have negative labels and the last $n_2=n-n_1$ inputs have positive labels. Then we compute ranking scores on first $n_1$ inputs using a ranking model $h$ and denote these ranking scores as $\{ r_i\}_{i=1}^{n_1}$, where $r_i = h(x_i)$. 
Given a miscoverage rate tolerance $\alpha$, the desired binary classifier is constructed:
\begin{equation*}
    \hat{\phi}(x) = 2 \cdot \mathds{1}(h(x)>r_{(v)}) - 1, 
\end{equation*}
where $v = \lceil (n_1  +1) (1-\alpha) \rceil$, $r_{(v)}$ is the $v$-th order statistic of $\{ r_i \}_{i=1}^{n_1}$, that is, $r_{(1)} \le \cdots \le r_{(n_1)}$. 
We note under \Cref{assump:h-continuous}, the ranking model $h$ is continuous, thus $\mathbb{P}(r_i = r_j) = 0, \forall i,j \in [n_1]$ and the strict inequalities hold $r_{(1)} < \cdots < r_{(n_1)}$.
In the following~\Cref{prop:miscoverage-rate-control}, we show the miscoverage rate of our calibrator is well-controlled by proving that the \mbox{Type I error} of this constructed binary classifier is well-controlled.

\begin{proposition}[Miscoverage rate control]
\label{prop:miscoverage-rate-control}
Under~\Cref{assump:h-continuous}, given a finite test data set of size $m$ that is i.i.d. drawn from $\mathbb{P}(X,Y)$, the empirical miscoverage rate $\hat{F}_0(g)$ of $g$ on $D$ is well-controlled:
\begin{align}
\mathbb{P}\left(\left| \hat{F}_{0}(g) - F_{0}(g) \right| \ge \delta\right) & \approx \mathbb{P}\left(\left| R_{0} - F_{0}(g) \right| \ge \delta\right) \\ \nonumber
& \le 2 \exp \left( -\frac{\delta^{2}}{2 \sigma^{2}} \right),
\end{align}
where $F_{0}(g)$ is the population miscoverage rate of $g$, the random variable $R_{0}$ has a Normal distribution $\mathcal{N}(F_{0}(g), \sigma^{2})$, $\sigma^{2} = F_{0}(g) (1- F_{0}(g)) / m_1$, $m_1 \le m$ is the number of samples whose predictions are correct.

Further, we have~\cite{tong2018}:
\begin{equation}
\label{eq:tong2018}
    \mathbb{P}(F_{0}(g) > \alpha) = \sum_{j=v}^{n_1} \binom{n_1}{j} (1-\alpha)^j \alpha^{n_1 - j},
\end{equation}
where $v = \lceil (n_1  +1) (1-\alpha) \rceil$.
\end{proposition}

\Cref{eq:tong2018} is given in~\citet{tong2018} and we adapt it here for our purpose. Essentially it means the population miscoverage rate is smaller than the predefined miscoverage rate tolerance with a high probability. A rough estimation of the miscoverage rate of $g$ is $\lfloor (1+n) \alpha \rfloor / (1+n_1) \le \alpha$. Although the miscoverage rate of our constructed calibration map can be well-controlled using the above constructed binary classifier, there is no such a guarantee that the \mbox{Type II error} of $\hat{\phi}$ is also well-controlled. 
Combining the results shown in \Cref{prop:ece-naive-binary-classifier}, it can be seen that the empirical ECE depends solely on $\hat{R}_1(\phi)$, since $\hat{\pi}_0$ is a constant given a finite data set. Thus we cannot make sure the empirical calibration error is small if we keep using the construction rules shown in \Cref{prop:optimal-calibration}.

A simple refinement to the above issue is to utilize a separate calibration model $g_m \in \mathcal{G}$ and the updated construction rules for a calibration map are listed in the following:\\
\begin{itemize}
    \item if $\hat{\phi}(x)=-1$, we let $g(x) = g_m(x)$,\\
    \item if $\hat{\phi}(x)=+1$, we let $\max_{i \in [k]} g(x)_i = \frac{1}{k}$.
\end{itemize}

Now it is clear why we call our post-hoc calibration approach  \textit{Meta-Cal}, since our calibrator is built from a base calibration map. The rationality of the above updated rules is shown in~\Cref{prop:meta-cal-ece-lower-bound}.

\begin{proposition}[Lower bound of Meta-Cal]
\label{prop:meta-cal-ece-lower-bound}
Suppose the base calibration map $g_m \in \mathcal{G}_a$ and $g$ is constructed using the above updated rules, then:
\begin{equation}
\label{eq:lower-bound-of-meta-cal}
\sup_{B \in \mathcal{B}} \widehat{\ECE}_B(g \circ f )  > w \frac{1 - \hat{\pi}_0}{k},
\end{equation}
where $w = (1 - \hat{R}_0(\hat{\phi})) \hat{\pi}_0 + \hat{R}_1(\hat{\phi}) (1 - \hat{\pi}_0) < 1$, $\hat{R}_0(\hat{\phi})$ and $\hat{R}_1(\hat{\phi})$ are empirical \mbox{Type I error} and \mbox{Type II error} of $\hat{\phi}$ respectively, $\hat{\pi}_0$ is the empirical accuracy of $f$ on a finite data set. 
\end{proposition}

The result in~\Cref{prop:meta-cal-ece-lower-bound} should be compared to the lower bound shown in~\Cref{prop:acc-pre-ece-lower-bound}. It can be seen that the calibration map constructed by Meta-Cal has an improved calibration error lower bound, compared with its accuracy-preserving base calibrator.~\Cref{prop:meta-cal-ece-lower-bound} should also be compared with~\Cref{prop:ece-naive-binary-classifier}. From~\Cref{eq:lower-bound-of-meta-cal}, using the decomposition $\frac{w}{k} = \frac{(1 - \hat{R}_0(\hat{\phi})) \hat{\pi}_0}{k} + \frac{\hat{R}_1(\hat{\phi}) (1 - \hat{\pi}_0)}{k}$, we see the influence of the \mbox{Type II error} has been effectively diluted. This is desirable since we only have control over the \mbox{Type II error} of $\hat{\phi}$.

At the same time, the miscoverage rate of our constructed calibrator is well-controlled. In~\Cref{sec:experiment}, we empirically show that Meta-Cal outperforms other competing methods in terms of the empirical estimation of ECE.~\Cref{algo:meta-cal-miscoverage-control} sketches the construction of Meta-Cal under a miscoverage rate constraint. For details of this construction, please see the first paragraph of~\Cref{sec:method:miscoverage-control}.

\begin{algorithm}
\caption{Meta-Cal (miscoverage control)}
\label{algo:meta-cal-miscoverage-control}
\begin{algorithmic}[1]
\STATE {\bfseries Input:} Training data set $\{ (x_i,y_i) \}_{i=1}^{n}$, miscoverage rate tolerance $\alpha$, base calibration model $g_m$, ranking model $h$.
\STATE {\bfseries Output:} Binary classifier $\hat{\phi}$, Meta-Cal calibration model $g$.
\STATE Partition the training data set randomly into two parts. The first part has only negative ($\hat{Y}=Y$) samples. The second part contains both negative and positive samples ($\hat{Y} \neq Y$).
\STATE Compute ranking scores on the first part using the ranking model $h$. Compute threshold $r^*$ based on $\alpha$.
\STATE Construct a binary classifier $\hat{\phi}$ based on $r^*$.
\STATE Train a base calibration model $g_m$ using samples whose scores are smaller than $r^*$ among the second part.
\STATE Construct the final calibration map $g$ using updated rules.
\end{algorithmic}
\end{algorithm}

\subsection{Coverage Accuracy Control}
\label{sec:method:coverage-accuracy-control}

In some applications, we are more interested in controlling the coverage accuracy of a calibration map instead of its miscoverage rate. 
To construct a calibration map with a controlled coverage accuracy, similar to~\Cref{sec:method:miscoverage-control}, we need to build a suitable binary classifier.
It will be convenient to introduce the concept of a calibrated binary CPE model~\cite{narasimhan2013}. 


\begin{definition}[Calibrated binary CPE model~\cite{narasimhan2013}]
\label{def:calibrated-cpe-model}
A binary class probability estimation (CPE) model $\hat{\eta}: \mathcal{X} \to [0,1]$ is said to be calibrated w.r.t. a probability $\mathbb{P}(X,\epsilon)$ on $\mathcal{X} \times \{-1, +1 \}$ if:
\begin{equation*}
    \mathbb{P}(\epsilon=1 \mid \hat{\eta}(x) = u) = u, \forall u \in \text{range}(\eta)
\end{equation*}
where $\text{range}(\hat{\eta})$ denotes the range of $\hat{\eta}$.
\end{definition}

We note a calibrated binary CPE model is different from a perfectly calibrated model (\Cref{def:calibration-err}) in the binary setting. Before we state our bound for the coverage accuracy, we present two existence propositions.

\begin{proposition}[Existence of a monotonically increasing CPE transformation]
\label{prop:existence-mono-inc-cpe}
Under~\Cref{assump:h-continuous,assump:eta-square-integral}, let $h: \mathcal{X} \to [a,b]$ (where $a,b \in \mathbb{R}$, $a < b$) be any bounded-range ranking model. There exists a monotonically increasing function $t: \mathbb{R} \to [0,1]$ such that $t \circ h$ resulting from composing $t$ and $h$ is a calibrated binary CPE model.
\end{proposition}

\begin{proof}
The existence of such a monotonically increasing CPE transformation directly follows by combining the results of \citet[Lemma 13]{narasimhan2013} and \citet[Definition 12]{narasimhan2013}.
\end{proof}

\begin{proposition}[Existence of a monotonically decreasing coverage accuracy transformation]
\label{prop:existence-mono-dec-cov-acc}
Under~\Cref{assump:h-continuous,assump:eta-square-integral}, let $h: \mathcal{X} \to [a,b]$ (where $a,b \in \mathbb{R}$, $a < b$) be any bounded-range ranking model. There exists a monotonically decreasing function $l: \mathbb{R} \to [0,1]$ such that $l(s)$ is the coverage accuracy of a calibration map $g$ which is constructed using the binary classifier $\hat{\phi}(x) = 2 \cdot \mathds{1}(h(x) > s) - 1$.
\end{proposition}
\begin{proof}
From~\Cref{prop:existence-mono-inc-cpe}, there exists a monotonically increasing $t$ such that $t \circ h$ is a calibrated binary CPE model. 
Let $\mathcal{E}_h(s) = \mathbb{P}(\epsilon=1 \mid h(x) < s)$, since $t$ is monotonically increasing, we have:
\begin{equation*}
\mathcal{E}_h(s) = \mathbb{P}(\epsilon = 1 \mid t(h(x)) < t(s)).
\end{equation*}
Because $t \circ h$ is a calibrated binary CPE model, using~\Cref{def:calibrated-cpe-model}, we further have:
\begin{equation*}
\mathcal{E}_h(s) = \int_{c}^{t(s)} u du = \frac{1}{2} t(s)^2 - \text{constant}
\end{equation*}
where $c$ is a constant.

From~\Cref{def:coverage-accuracy}, the desired coverage accuracy transformation is:
\begin{equation*}
    l(s) = 1 - \mathcal{E}_h(s) = \text{constant} - \frac{1}{2} t(s)^2.
\end{equation*}
Since $t(s) \ge 0, \forall s \in \mathbb{R}$ and $t(\cdot)$ is a monotonically increasing function, it is obvious that $l$ is a monotonically decreasing function. From the definition of $\mathcal{E}_h(s)$, the binary classifier that is used to construct $g$ is exactly $\hat{\phi}(x) = 2 \cdot \mathds{1}(h(x) > s) - 1$.
\end{proof}

Based on the above two existence propositions, we describe a bound for the empirical coverage accuracy in~\Cref{prop:risk-control}.

\begin{table*}

\caption{\label{tab:ece-comparison}ECE comparison.  \textit{Uncal}, \textit{TS}, \textit{ETS}, \textit{GPC}, \textit{MetaMis}, \textit{MetaAcc}  
denote no-calibration, temperature scaling, ensemble temperature scaling, Gaussian Process calibration, Meta-Cal under miscoverage rate constraint and Meta-Cal under coverage accuracy constraint respectively. Reported values are the average of 40 independent runs. All standard errors are less than $5e-4$.}
\centering
\begin{tabular}[t]{>{}lllllllll}
\toprule
\textbf{Dataset} & \textbf{Network} & \textbf{Acc} & \textbf{Uncal} & \textbf{TS} & \textbf{ETS} & \textbf{GPC} & \textbf{MetaMis} & \textbf{MetaAcc}\\
\midrule
 & DenseNet40 & 0.9242 & 0.05105 & 0.00510 & 0.00567 & 0.00634 & 0.00434 & 0.00355\\
\cmidrule{2-9}
 & ResNet110 & 0.9356 & 0.04475 & 0.00781 & 0.00809 & 0.00684 & 0.00391 & 0.00441\\
\cmidrule{2-9}
 & ResNet110SD & 0.9404 & 0.04022 & 0.00439 & 0.00509 & 0.00364 & 0.00350 & 0.00315\\
\cmidrule{2-9}
\multirow{-4}{*}{\raggedright\arraybackslash \textbf{CIFAR10}} & WideResNet32 & 0.9393 & 0.04396 & 0.00706 & 0.00712 & 0.00684 & 0.00485 & 0.00532\\
\cmidrule{1-9}
 & DenseNet40 & 0.7000 & 0.21107 & 0.01067 & 0.01104 & 0.01298 & 0.01093 & 0.00793\\
\cmidrule{2-9}
 & ResNet110 & 0.7148 & 0.18182 & 0.02037 & 0.02130 & 0.01348 & 0.01815 & 0.01441\\
\cmidrule{2-9}
 & ResNet110SD & 0.7283 & 0.15496 & 0.01043 & 0.01057 & 0.01265 & 0.01109 & 0.00733\\
\cmidrule{2-9}
\multirow{-4}{*}{\raggedright\arraybackslash \textbf{CIFAR100}} & WideResNet32 & 0.7382 & 0.18425 & 0.01332 & 0.01351 & 0.00993 & 0.01332 & 0.01189\\
\cmidrule{1-9}
 & DenseNet161 & 0.7705 & 0.05531 & 0.02053 & 0.02064 & NA & 0.01388 & 0.01248\\
\cmidrule{2-9}
\multirow{-2}{*}{\raggedright\arraybackslash \textbf{ImageNet}} & ResNet152 & 0.7620 & 0.06290 & 0.02023 & 0.02004 & NA & 0.01360 & 0.01138\\
\bottomrule
\end{tabular}
\end{table*}

\begin{proposition}[Coverage accuracy control]
\label{prop:risk-control}
Under~\Cref{assump:h-continuous,assump:eta-square-integral}, the empirical coverage accuracy $\hat{F}_1(g)$ of $g$ on a finite test data set of size $m$ that is drawn  i.i.d.\ from $\mathbb{P}(X,Y)$ is well-controlled:

\begin{align}
\mathbb{P}\left(\left| \hat{F}_{1}(g) - \beta \right| \ge \delta\right) & \approx \mathbb{P}\left(\left| R_{1} - \beta \right| \ge \delta\right) \nonumber\\
& \le 2 \exp \left( -\frac{m_{1} \delta^{2}}{2 \beta (1 - \beta)}  \right),
\end{align}
where $\beta$ is the desired coverage accuracy, the random variable $R_1$ has a Normal distributon $\mathcal{N}\left(\beta, \frac{\beta (1-\beta)}{m_1}\right)$, and $m_1 \le m$ is the number of samples whose ranking scores are smaller than $l^{-1}(\beta)$ in the test data set. 

Further, the desired binary classifier is:
\begin{equation}
\label{eq:risk-control-binary-classifier}
    \hat{\phi}(x) = 2 \cdot \mathds{1}(h(x) > l^{-1}(\beta)) - 1.
\end{equation}
\end{proposition}

\begin{algorithm}[t]
\caption{Meta-Cal (coverage accuracy control)}
\label{algo:meta-cal-coverage-accuracy-control}
\begin{algorithmic}[1]
\STATE {\bfseries Input:} Training data set $\{ (x_i,y_i) \}_{i=1}^{n}$, desired coverage accuracy $\beta$, base calibration model $g_m$, ranking model $h$.
\STATE {\bfseries Output:} Binary classifier $\hat{\phi}$, Meta-Cal calibration model $g$.
\STATE Randomly split the training data set into two parts.
\STATE Estimate the coverage accuracy transformation $\hat{l}$ on the first part.
\STATE Compute a threshold $r^*=\hat{l}^{-1}(\beta)$ based on the estimated $\hat{l}$ and $\beta$.
\STATE Construct a binary classifier $\hat{\phi}$ based on $r^*$.
\STATE Train a base calibration model $g_m$ using samples among the second part whose scores are smaller than $r^*$.
\STATE Construct the final calibration map $g$ using updated rules.
\end{algorithmic}
\end{algorithm}

It should be noted that among all calibration maps whose coverage accuracy is higher than $\beta$, the calibration map constructed using the binary classifier in~\Cref{eq:risk-control-binary-classifier} has the smallest miscoverage rate.
In the following, we describe how to estimate the monotonically decreasing function $l$ in practice. Given a finite calibration data set $\{(x_i,y_i)\}_{i=1}^{n}$ that is drawn i.i.d.\ from the joint distribution $\mathbb{P}(X,Y)$ on $\mathcal{X} \times \mathcal{Y}$: 
Firstly the ranking scores $\{ r_i \}_{i=1}^{n}$, where $r_i = h(x_i)$, are computed using our ranking model $h$.
Then a set of bins $\{ I_1,\cdots,I_b \}$ to partition $\{ r_i \}_{i=1}^{n}$ is constructed through uniform mass binning~\cite{zadrozny2001a}. For a given $j \in \{ 1,\cdots,b \}$, let $l_j^{(s)}$ and $l_j^{(a)}$ denote the average of ranking scores on $I_j$ and $f$'s accuracy  on $\{ I_1,\cdots,I_j \}$ respectively. 
Finally a decreasing isotonic regression model is fitted on $\{ (l_j^{(s)}, l_j^{(a)}) \}_{j=1}^{b}$ and the fitted model is an estimation of $l$. ~\Cref{algo:meta-cal-coverage-accuracy-control} sketches the construction of Meta-Cal under the coverage accuracy constraint. 

Several remarks are made to conclude this section. 
\begin{inparaenum}[(i)]
\item The miscoverage rate bound holds for all base calibration maps. The coverage accuracy bound requires the base calibrator to preserve the accuracy. 
\item These two bounds do not depend on the metric used to evaluate the level of calibration, no matter whether it is top-label calibration error (which we use in this paper) or marginal calibration error~\cite{kumar2019,kull2019}.
\item To ensure the independence assumption, the training data of Meta-Cal should be different from the data set used to train the multi-class classifier.
\end{inparaenum}

\section{Experiments}
\label{sec:experiment}

In this section, we present empirical results on a set of multi-class classifiers. Firstly, we compare our our proposed Meta-Cal described in~\Cref{sec:method} with several baselines. Then we validate the two constraints investigated in this work are well satisfied in~\Cref{sec:exp:verifying-constraints}.

\begin{figure*}[ht]
\centering
\includegraphics[width=\linewidth]{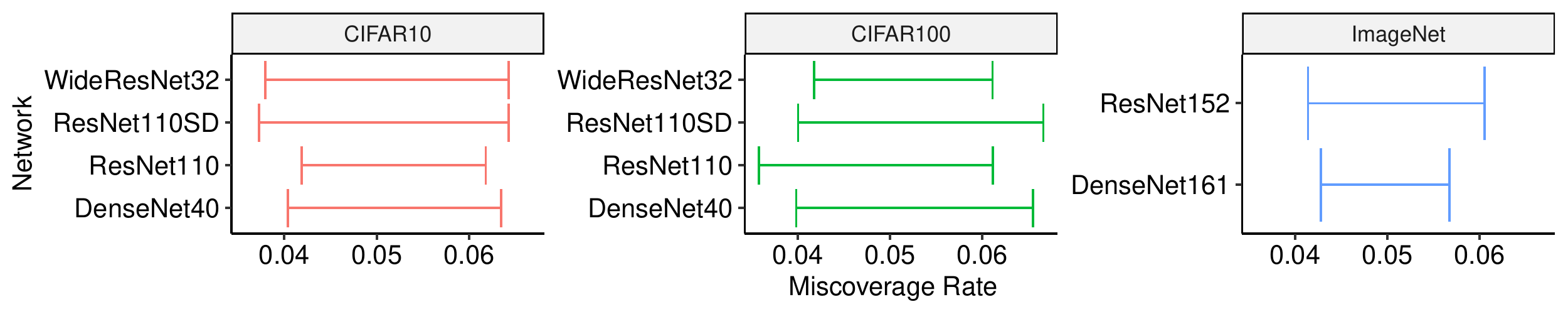}
\caption{Empirical miscoverage rate. The error bars show $\pm 2$ standard deviation of 40 independent runs. The desired miscoverage rates for CIFAR-10, CIFAR-100 and ImageNet are all set to be $0.05$.}
\label{fig:verify-miscoverage-rate}
\end{figure*}

\begin{figure*}[ht]
\centering
\includegraphics[width=\linewidth]{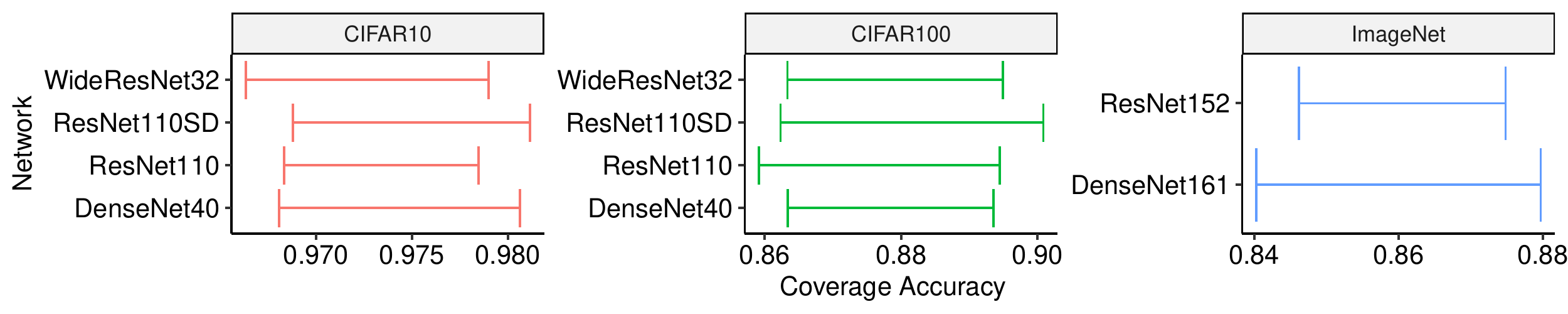}
\caption{Empirical coverage accuracy. The error bars show $\pm 2$ standard deviation of 40 independent runs. The desired coverage accuracy for CIFAR-10, CIFAR-100 and ImageNet are set to be $0.97$, $0.87$ and $0.85$, respectively.}
\label{fig:verify-coverage-accuracy}
\end{figure*}

\subsection{Calibrating Neural Networks}
\label{sec:exp:calibrate-nn}

In this section, we compare our approach with other methods of post-hoc multi-class calibration. 
Following previous works~\cite{guo2017,kull2019,zhang2020,wenger2020}, we report calibration results on several neural network classifiers trained on CIFAR-10, CIFAR-100~\cite{Krizhevsky09learningmultiple} and ImageNet~\cite{deng2009imagenet}.
For CIFAR-10 and CIFAR-100, the following networks are used: DenseNet~\cite{huang2016}, ResNet~\cite{he2015}, ResNet with stochastic depth~\cite{huang2016b}, WideResNet~\cite{zagoruyko2016}. 45000 out of 60000 images are used for training these classifiers. The remaining 15000 images are held out for training and evaluating post-hoc calibration methods. The training details are given in~\Cref{supp:train-detials}. These 15000 samples are randomly split into 5000/10000 samples to train and evaluate a post-hoc calibration method.
For ImageNet, we use pre-trained DenseNet-161 and ResNet-152 from PyTorch~\cite{paszke2019}. 50000 images in the validation set are used for training and evaluating post-hoc calibration methods. To train and test a calibration map, we randomly split these samples into 25000/25000 images.


The following post-hoc algorithms are used for comparison: temperature scaling~\cite{guo2017}, ensemble temperature scaling~\cite{zhang2020}, Gaussian Process calibration~\cite{wenger2020}. 
The Dirichlet calibration~\cite{kull2019} is not included in our empirical comparison, since on the neural network classifiers we considered in this paper, its performance is consistently worse than other methods~\cite{zhang2020}.

To construct a calibration map using our proposed Meta-Cal, firstly we need to decide which base calibrator to use. Throughout the experimental part of this paper, we use temperature scaling since: (i) it is the most computationally efficient approach among the above baselines, and (ii) it can be seen from~\Cref{prop:meta-cal-ece-lower-bound} that Meta-Cal augments the performance of temperature scaling, since it is an accuracy-preserving calibration map.
Secondly a ranking model is required to construct a desired binary classifier based on different constraints. In this paper, we define the ranking model as the entropy of the output of an uncalibrated probabilistic classifier, that is, given $X \in \mathcal{X}$, we define $h(X)=-\sum_{i \in [k]} f(X)_i \log f(X)_i$. An alternative way is to learn a ranking model using a consistent ranking algorithm~\cite{clemencon2008} on a separate data set. 
A potential advantage of this approach is the constructed binary classifier has a lower \mbox{Type II error}, thus from~\Cref{prop:ece-naive-binary-classifier}, the worst-case estimation of expected calibration error can be improved. A disadvantage of this approach is a separate data set is required to learn a ranking model and this effectively means we have less samples to train a calibration model.

The experimental configurations specific to our proposed approach are as follows. For Meta-Cal under the miscoverage rate constraint, we set the miscoverage rate tolerance to be $0.05$ for all neural network classifiers and all data sets used in the experiments. For Meta-Cal under the coverage accuracy constraint, we set the desired coverage accuracy to be $0.97$, $0.87$, $0.85$ for CIFAR-10, CIFAR-100 and ImageNet, respectively. For reference, the original accuracy for each configuration is shown in the third column in~\Cref{tab:ece-comparison}.
In both settings, we randomly select $1/10$ samples (up to 500 samples) from the calibration data set to construct a binary classifier or estimate the coverage accuracy transformation function.

\begin{table*}

\caption{\label{tab:time-comparison}Time comparison (s). Reported values are the average of 40 independent runs.}
\centering
\begin{tabular}[t]{>{}llllllll}
\toprule
\textbf{Dataset} & \textbf{Network} & \textbf{Uncal} & \textbf{TS} & \textbf{ETS} & \textbf{GPC} & \textbf{MetaMis} & \textbf{MetaAcc}\\
\midrule
 & DenseNet40 & 0.004 & 0.074 & 0.143 & 4.722 & 0.102 & 0.094\\
\cmidrule{2-8}
 & ResNet110 & 0.004 & 0.066 & 0.050 & 5.255 & 0.127 & 0.084\\
\cmidrule{2-8}
 & ResNet110SD & 0.004 & 0.075 & 0.148 & 5.401 & 0.101 & 0.090\\
\cmidrule{2-8}
\multirow{-4}{*}{\raggedright\arraybackslash \textbf{CIFAR10}} & WideResNet32 & 0.004 & 0.078 & 0.133 & 5.116 & 0.105 & 0.093\\
\cmidrule{1-8}
 & DenseNet40 & 0.014 & 0.492 & 0.492 & 40.218 & 0.470 & 0.336\\
\cmidrule{2-8}
 & ResNet110 & 0.026 & 0.426 & 1.404 & 42.118 & 0.408 & 0.299\\
\cmidrule{2-8}
 & ResNet110SD & 0.015 & 0.357 & 0.931 & 40.867 & 0.421 & 0.327\\
\cmidrule{2-8}
\multirow{-4}{*}{\raggedright\arraybackslash \textbf{CIFAR100}} & WideResNet32 & 0.014 & 0.431 & 0.708 & 44.882 & 0.467 & 0.353\\
\cmidrule{1-8}
 & DenseNet161 & 0.304 & 12.707 & 99.488 & NA & 13.573 & 9.903\\
\cmidrule{2-8}
\multirow{-2}{*}{\raggedright\arraybackslash \textbf{ImageNet}} & ResNet152 & 0.309 & 15.624 & 111.589 & NA & 12.581 & 8.426\\
\bottomrule
\end{tabular}
\end{table*}

The comparison results of different calibration methods are shown in \Cref{tab:ece-comparison}.  
The empirical ECE are esitmated with 15 equal-width bins and ECE values reported in~\Cref{tab:ece-comparison} are the average of 40 independent runs. 
Since GPC does not converge on ImageNet, its results on ImageNet are marked as NA in~\Cref{tab:ece-comparison}. We note CE loss instead of MSE in optimizing ETS.


From~\Cref{tab:ece-comparison}, it can be seen that our proposed method Meta-Cal performs much better w.r.t. the empirical ECE in almost all configurations. It should be noted that~\Cref{tab:ece-comparison} should be interpreted in a careful way, since our proposed method works in a novel setting, that is, post-hoc calibration under constraints. Essentially, we can interpret the improved calibration is due to the miscoverage price we pay. It is worth noting that after calibration, GPC has an around $0.2\%$ accuracy drop in ResNet110 and ResNet110SD on CIFAR-10 data set. It is possible that GPC is trading off accuracy for calibration in an implicit and uncontrolled way. ~\Cref{tab:time-comparison} shows the running time of different calibration methods. It can be seen Meta-Cal has little overhead over its base calibrator (TS in our case). Sometimes Meta-Cal takes less than time than TS because its base calibrator is optimized on a subset of the whole calibration data set.

\subsection{Verifying Constraints}
\label{sec:exp:verifying-constraints}

In this section, we empirically verify whether constraints of miscoverage rate and coverage accuracy are satisfied. The purpose of this section is to support our theoretical claims in~\Cref{sec:method:miscoverage-control} and~\Cref{sec:method:coverage-accuracy-control}.

The error bar plot in \Cref{fig:verify-miscoverage-rate} depicts $\pm 2$ standard deviations of the empirical miscoverage rate over 40 independent runs. It can be seen from~\Cref{fig:verify-miscoverage-rate} that the miscoverage rate is well-controlled given a miscoverage rate tolerance equal to $0.05$. In \Cref{fig:verify-coverage-accuracy}, the $\pm 2$ standard deviation of the empirical coverage accuracy is illustrated. Given the desired coverage accuracies are set to be $0.97$, $0.87$ and $0.85$ for CIFAR-10, CIFAR-100 and ImageNet, respectively, it can be seen that the coverage accuracy is well-controlled.



\section{Conclusion}
\label{sec:conclusion}

In this work, we propose Meta-Cal: a post-hoc calibration method for multi-class classification under practical constraints, including miscoverage rate and coverage accuracy.
Our proposed approach can augment an accuracy-preserving calibration map and improve its calibration performance. Contrary to post-hoc calibration methods that cannot preserve accuracy, our approach has full control over coverage accuracy. 
A series of theoretical results are given to show the validity of Meta-Cal. 
Empirical results on a range of neural network classifiers on several popular computer vision data sets show our approach is able to improve an existing calibration method. 

\subsection*{Acknowledgements}
Xingchen Ma is supported by Onfido.
This research received funding from the Flemish Government under the “Onderzoeksprogramma Artificiële Intelligentie (AI) Vlaanderen” programme.

\bibliography{mybib}


\includepdf[pages=-]{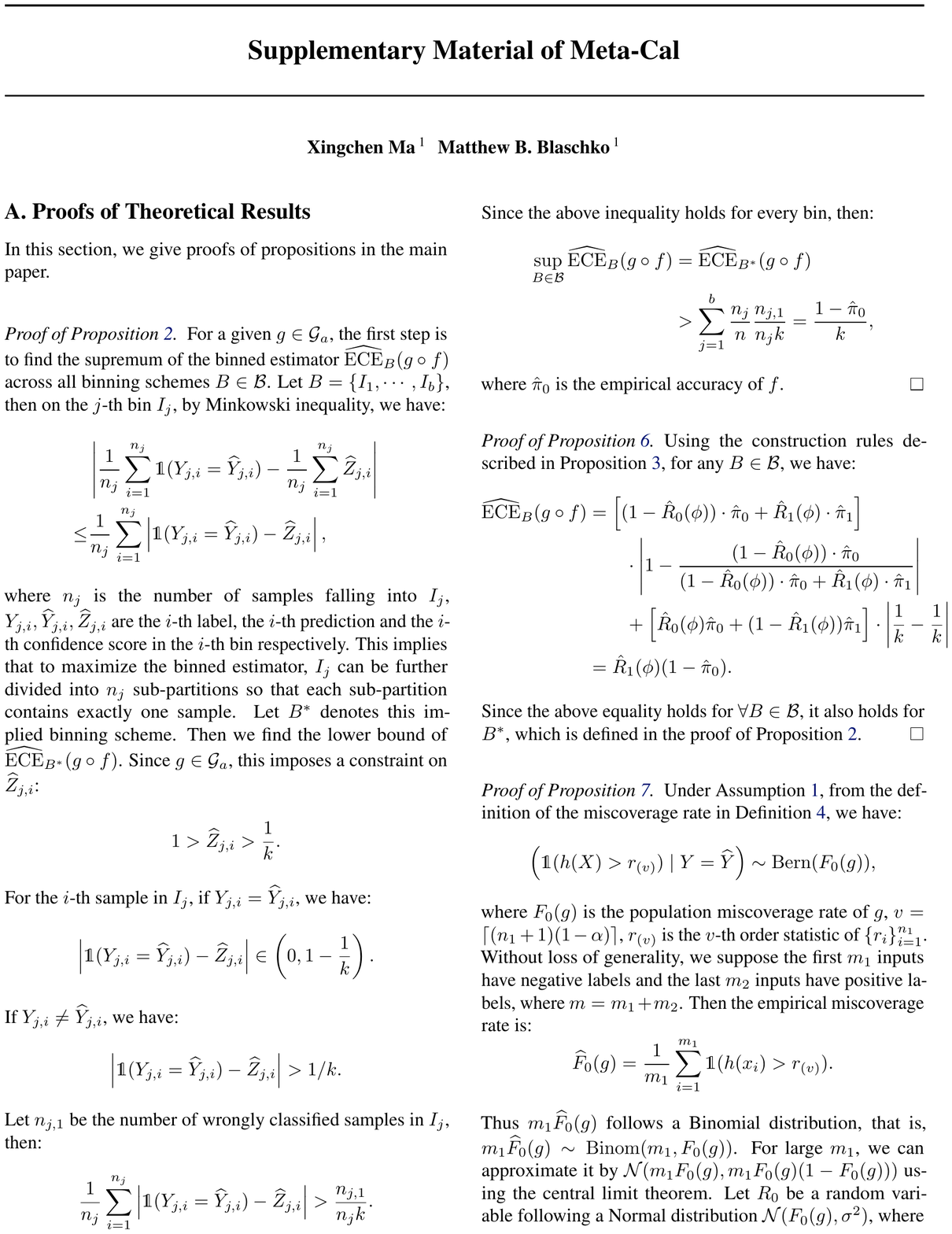}
\end{document}


\twocolumn[
\icmltitle{Supplementary Material of Meta-Cal}

\begin{icmlauthorlist}
\icmlauthor{Xingchen Ma}{kul}
\icmlauthor{Matthew B.\ Blaschko}{kul}
\end{icmlauthorlist}

\icmlaffiliation{kul}{ESAT-PSI, KU Leuven, Belgium}

\icmlcorrespondingauthor{Xingchen Ma}{xingchen.ma@esat.kuleuven.be}

\icmlkeywords{Calibration, Ranking}

\vskip 0.3in
]

\appendix

\section{Proofs of Theoretical Results}
\label[supp]{supp:proofs}
In this section, we give proofs of propositions in the main paper.

\begin{proof}[Proof of~\Cref{prop:acc-pre-ece-lower-bound}]
For a given $g \in \mathcal{G}_a$, the first step is to find the supremum of the binned estimator $\widehat{\ECE}_B(g \circ f)$ across all binning schemes $B \in \mathcal{B}$. Let $B = \{ I_1,\cdots,I_b \}$, then on the $j$-th bin $I_j$, by Minkowski inequality, we have:
\begin{align*}
& \left| \frac{1}{n_{j}} \sum_{i=1}^{n_{j}} \mathds{1}(Y_{j,i}=\widehat{Y}_{j,i}) - \frac{1}{n_{j}} \sum_{i=1}^{n_{j}} \widehat{Z}_{j,i} \right| \\
\le & \frac{1}{n_{j}} \sum_{i=1}^{n_{j}}  \left| \mathds{1}(Y_{j,i}=\widehat{Y}_{j,i}) -  \widehat{Z}_{j,i} \right|,
\end{align*}
where $n_{j}$ is the number of samples falling into $I_{j}$, $Y_{j,i},
\widehat{Y}_{j,i}, \widehat{Z}_{j,i}$ are the $i$-th label, the $i$-th
prediction and the $i$-th confidence score in the $i$-th bin respectively. 
This implies that to maximize the binned estimator, $I_j$ can be further divided into $n_j$ sub-partitions so that each sub-partition contains exactly one sample. Let $B^*$ denotes this implied binning scheme.
Then we find the lower bound of $\widehat{\ECE}_{B^*}(g \circ f)$. 
Since $g \in \mathcal{G}_a$, this imposes a constraint on $\widehat{Z}_{j,i}$:
\begin{align*}
1 > \widehat{Z}_{j,i} > \frac{1}{k}.
\end{align*}
For the $i$-th sample in $I_j$, if $Y_{j,i}=\widehat{Y}_{j,i}$, we have:
\begin{align*}
\left| \mathds{1}(Y_{j,i}=\widehat{Y}_{j,i}) -  \widehat{Z}_{j,i} \right| \in \left( 0, 1 - \frac{1}{k}  \right).
\end{align*}
If $Y_{j,i} \neq \widehat{Y}_{j,i}$, we have:
\begin{align*}
\left| \mathds{1}(Y_{j,i}=\widehat{Y}_{j,i}) -  \widehat{Z}_{j,i} \right| > 1 / k.
\end{align*}
Let $n_{j,1}$ be the number of wrongly classified samples in $I_j$, then:
\begin{equation*}
    \frac{1}{n_{j}} \sum_{i=1}^{n_{j}}  \left| \mathds{1}(Y_{j,i}=\widehat{Y}_{j,i}) -  \widehat{Z}_{j,i} \right| > \frac{n_{j,1}}{n_j k}.  
\end{equation*}

Since the above inequality holds for every bin, then:
\begin{align*}
\sup_{B \in \mathcal{B}} \widehat{\ECE}_B(g \circ f ) & = \widehat{\ECE}_{B^*}(g \circ f) \\
& > \sum_{j=1}^{b} \frac{n_j}{n} \frac{n_{j,1}}{n_j k}  = \frac{1 - \hat{\pi}_0}{k},
\end{align*}
where $\hat{\pi}_0$ is the empirical accuracy of $f$.
\end{proof}

\begin{proof}[Proof of \Cref{prop:ece-naive-binary-classifier}]
Using the construction rules described in~\Cref{prop:optimal-calibration}, for any $B \in \mathcal{B}$, we have:
\begin{align*}
\widehat{\ECE}_B(g \circ f ) & = \left[ (1 - \hat{R}_{0}(\phi)) \cdot \hat{\pi}_{0} + \hat{R}_{1}(\phi) \cdot \hat{\pi}_{1}  \right] \\
&\qquad \cdot \left| 1 - \frac{(1 - \hat{R}_{0}(\phi)) \cdot \hat{\pi}_{0}}{(1 - \hat{R}_{0}(\phi)) \cdot \hat{\pi}_{0} + \hat{R}_{1}(\phi) \cdot \hat{\pi}_{1} } \right| \\
&\qquad + \left[ \hat{R}_{0}(\phi) \hat{\pi}_{0} + (1 - \hat{R}_{1}(\phi)) \hat{\pi}_{1} \right] \cdot \left| \frac{1}{k} - \frac{1}{k} \right| \\
& = \hat{R}_{1}(\phi) (1 - \hat{\pi}_{0}).
\end{align*}
Since the above equality holds for $\forall B \in \mathcal{B}$, it also holds for $B^*$, which is defined in the proof of~\Cref{prop:acc-pre-ece-lower-bound}.
\end{proof}

\begin{proof}[Proof of~\Cref{prop:miscoverage-rate-control}] Under~\Cref{assump:h-continuous}, from the definition of the miscoverage rate in~\Cref{def:miscoverage-rate}, we have: 
\begin{equation*}
    \left( \mathds{1}(h(X) > r_{(v)}) \mid Y=\widehat{Y} \right) \sim \Bern (F_{0}(g)),
\end{equation*}
where $F_{0}(g)$ is the population miscoverage rate of $g$, $v = \lceil (n_1  +1) (1-\alpha) \rceil$, $r_{(v)}$ is the $v$-th order statistic of $\{ r_i \}_{i=1}^{n_1}$. Without loss of generality, we suppose the first $m_1$ inputs have negative labels and the last $m_2$ inputs have positive labels, where $m=m_1 + m_2$. Then the empirical miscoverage rate is:
\begin{equation*}
    \widehat{F}_{0}(g) = \frac{1}{m_1} \sum_{i=1}^{m_1} \mathds{1}(h(x_i) > r_{(v)}).
\end{equation*}
Thus $m_1 \widehat{F}_{0}(g)$ follows a Binomial distribution, that is, $m_1 \widehat{F}_{0}(g) \sim \Binom(m_1, F_{0}(g))$. For large $m_1$, we can approximate it by $\mathcal{N}(m_1 F_{0}(g), m_1 F_{0}(g) (1 - F_{0}(g)))$ using the central limit theorem. 
Let $R_0$ be a random variable following a Normal distribution $\mathcal{N}(F_{0}(g), \sigma^{2})$, where $\sigma^{2} = F_{0}(g) (1- F_{0}(g)) / m_1$. 
Applying the Chernoff bound for a Gaussian variable:
\begin{equation*}
    \mathbb{P}\left(\left| R_{0} - F_{0}(g) \right| \ge \delta\right) \le 2 \exp \left( -\frac{\delta^{2}}{2 \sigma^{2}} \right).
\end{equation*}
Combining the above results finishes the proof of~\Cref{prop:miscoverage-rate-control}.
\end{proof}

\begin{proof}[Proof of~\Cref{prop:meta-cal-ece-lower-bound}]
This follows directly from~\Cref{prop:acc-pre-ece-lower-bound}.
\end{proof}

\begin{proof}[Proof of~\Cref{prop:risk-control}]
Under~\Cref{assump:h-continuous,assump:eta-square-integral}, from the definition of coverage accuracy in~\Cref{def:coverage-accuracy} and~\Cref{prop:existence-mono-dec-cov-acc}, we have: 
\begin{equation*}
     \left( \mathds{1}(Y=\widehat{Y}) \mid h(X) < l^{-1}(\beta) \right) \sim \Bern (\beta).
\end{equation*}
Without loss of generality, we suppose ranking scores of the first $m_1 \le m$ inputs are smaller than $l^{-1}(\beta)$. The empirical coverage accuracy is:
\begin{equation*}
    \widehat{F}_{1}(g) = \frac{1}{m_1} \sum_{i=1}^{m_1} \mathds{1}(y_i = \hat{y}_i).
\end{equation*}
Thus $m_1 \widehat{F}_{1}(g)$ follows a Binomial distribution and can be approximated by a Gaussian distribution for large $m_1$. The remaining proof is similar to the proof of~\Cref{prop:miscoverage-rate-control}.
\end{proof}

\section{A Trivial Construction}
\label[supp]{supp:trivial-g}

In this section, we show bounding $\widehat{\ECE}_B(g \circ f)$ instead of its supremum makes little sense by giving a trivial construction. 
Suppose we are to evaluate the binned estimator on 4 samples and their predicted confidence scores, predictions, targets are $\{ 0.6, 0.7, 0.8, 0.81 \}$, $\{ 0,0,0,0\}$, $\{ 0,0,0,1\}$ respectively. Let the binning scheme be the interval $B = [0,1]$, we can construct a accuracy-preserving calibration map:
\begin{equation*}
g(x) = \left\{
\begin{array}{l@{\quad,\quad}l}
0.9 & \text{if } x=0.81 \\
x & \text{otherwise.}
\end{array}\right.
\end{equation*}
The binned estimator using such a calibration map on this data set is 0. Obviously, this calibration map is not practical.

\section{Training Details}
\label[supp]{supp:train-detials}

To train networks on CIFAR-10 and CIFAR-100, we use stochastic gradient descent with momentum (0.9) using mini-batches of 128 samples for 200 epochs. We also add a L2-weight decay regularization term, which is set to be 0.0001. The start learning rate is set to be 0.1 and is decreased to 0.01 and 0.001 in the beginning of the 80-th epoch and 150-th epoch respectively. 
Horizontal flipping and cropping are used as data augmentation.
